\documentclass[letterpaper, 10 pt, conference]{ieeeconf}  

\IEEEoverridecommandlockouts                              
\RequirePackage{algorithm}
\RequirePackage{algorithmic}
\usepackage{microtype}
\usepackage{graphicx}
\usepackage[]{caption}
\usepackage{amsmath}
\usepackage{amsfonts}
\usepackage{amssymb}
\usepackage{centernot}
\usepackage{mathtools}
\usepackage{twoopt}
\usepackage{booktabs} 
\usepackage{bm}

\usepackage[scr=boondoxo,scrscaled=1.05]{mathalfa}

\usepackage{amsthm}

\usepackage{tikz}
\usetikzlibrary{shapes,arrows,automata}

\usepackage{hyperref}
\newcommand{\R}{{\mathbb R}}

\newcommand{\matvec}{\mbox{vec}}

\newcommand{\defeq}{\vcentcolon=}

\newcommand{\ones}{\mathbf{1}}

\DeclareMathOperator*{\argmin}{arg\,min}

\newcommand{\task}[1][i]{{\mathcal{T}_{#1}}} 
\newcommand{\tasks}{\{ \task \}_{i=1}^N}
\newcommand{\losst}[1][i]{{l_{#1}}}
\newcommand{\lossv}[1][i]{{l_{#1}^{\textrm{val}}}}
\newcommand{\tasktarget}{{\mathcal{T}_{\textrm{target}}}}
\newcommand{\lossttarget}{l_{\textrm{target}}}
\newcommand{\lossvtarget}{l_{\textrm{target}}^{\textrm{val}}}
\newcommand{\lossttargetit}{l_{\textrm{target}}^{(k)}}
\newcommand{\losstotal}{l^{\textrm{total}}}
\newcommand{\lossopt}{l^*}

\newcommand{\thetait}[2]{\theta_{#1}^{(#2)}}
\newcommand{\phit}[1]{\phi^{(#1)}}
\newcommand{\hist}[2]{S_{#1}^{(#2)}}
\newcommand{\grad}[2]{G_{#1}^{(#2)}}

\newcommand{\Alg}{\textup{\textbf{Opt}}}
\newcommand{\MetaAlg}{\textup{\textbf{MetaOpt}}}

\newtheoremstyle{mytheoremstyle} 
    {\topsep}                    
    {\topsep}                    
    {\itshape}                   
    {}                           
    {\scshape}                   
    {.}                          
    {.5em}                       
    {}  
\theoremstyle{mytheoremstyle}
\theoremstyle{plain}
\newtheorem{theorem}{Theorem}

\theoremstyle{remark}
\newtheorem{remark}{Remark}

\title{\LARGE \bf
    Meta-Learning Parameterized First-Order Optimizers using Differentiable Convex Optimization
}

\author{Tanmay Gautam, Samuel Pfrommer, Somayeh Sojoudi
    \thanks{All authors are with the Department of Electrical Engineering and Computer Sciences, University of California Berkeley, Berkeley, CA, 94720.
    {\tt\small tgautam23@berkeley.edu}; {\tt\small sam.pfrommer@berkeley.edu}*; {\tt\small sojoudi@berkeley.edu}}
}

\begin{document}

\maketitle
\thispagestyle{empty}
\pagestyle{empty}

\begin{abstract}
Conventional optimization methods in machine learning and controls rely heavily on first-order update rules. Selecting the right method and hyperparameters for a particular task often involves trial-and-error or practitioner intuition, motivating the field of meta-learning. We generalize a broad family of preexisting update rules by proposing a meta-learning framework in which the inner loop optimization step involves solving a differentiable convex optimization (DCO). We illustrate the theoretical appeal of this approach by showing that it enables one-step optimization of a family of linear least squares problems, given that the meta-learner has sufficient exposure to similar tasks. Various instantiations of the DCO update rule are compared to conventional optimizers on a range of illustrative experimental settings.
\end{abstract}

\section{Introduction}\label{sec:introduction}
First-order optimization methods underpin a wide range of modern control and machine learning (ML) techniques. The field of deep learning, including domains such as computer vision \cite{alexnet,simonyan2015deep}, natural language processing \cite{nlp_attention}, deep reinforcement learning \cite{Sutton+Barto:1998}, and robotics \cite{Snderhauf2018TheLA}, has yielded revolutionary results when trained with variants of gradient descent such as stochastic gradient descent (SGD) \cite{sgd} and Adam \cite{Kingma2015AdamAM}. Algorithms like projected and conditional gradient descent extend the class of first-order methods to accommodate problems with constraints such as matrix completion, or training well-posed implicit deep models \cite{ghaoui2020implicit,gautam2022sequential-long}.

While this proliferation of methods has facilitated rapid advances across the control and ML communities, designing update rules tailored to specific problems still remains a challenge. This challenge is exacerbated by the fact that different domains are tasked with solving distinct problem types. The deep learning community is, for instance, tasked with solving high-dimensional non-convex problems whereas the optimal control community often deals with constrained convex problems where the constraints encode restrictions on the state space and system dynamics. Moreover, even different problem instances within a particular problem class may require significantly varying update rules. As an example, within deep learning, effective hyperparameter (e.g. learning rate) selection for algorithms such as Adam and SGD is highly dependent on the underlying model that is to be trained. 

\subsection{Contributions}\label{ssec:contributions}
This work proposes a new data-driven approach for optimization algorithm design based on differentiable convex optimization (DCO). This approach enables the use of previous optimization experience to propose update rules that efficiently solve new optimization tasks sampled from the same underlying problem class. We start by introducing the notion of DCO as a means to parameterize optimizers within the meta-learning framework. We then propose an efficient instantiation of meta-training that can be leveraged by the DCO optimizer to learn appropriate meta-parameters. To illustrate the generality of the DCO meta-learning framework, we then formulate concrete differentiable quadratic optimizations to solve unconstrained optimization problems, namely, DCO Gradient (DCOG), DCO Momentum (DCOM) and DCO General Descent (DCOGD). These DCO instantiations are generalizations of existing first-order update rules, which in turn demonstrates that existing methods can be thought of special cases of the DCO meta-learning framework.

DCO also provides sufficient structure conducive to rigorous theoretical analysis for the meta-learning problem. We establish convergence guarantees for the DCOGD optimizer to the optimal first-order update rule that leads to one step convergence when considering a family of linear least squares (LS) problems. Finally, we illustrate the potential of our proposed DCO optimizer instantiations by comparing convergence speed with popular existing first-order methods on illustrative regression and system identification tasks.

\subsection{Related Works}\label{ssec:relatedworks}
\subsubsection{Meta-Learning}\label{sssec:meta}
Deep learning has been shown to be particularly performant in scenarios where there is an abundance of training data and computational resources \cite{alexnet, Goodfellow-et-al-2016, nlp_attention, silver2017mastering}. This, however, excludes many important applications where there is an inherent lack of data or where computation is very expensive. Meta-learning attempts to address this issue by gaining learning process experience on similarly structured tasks \cite{mtlsurvey}. This learning-to-learn paradigm is aligned with the human and animal learning process which tends to improve with greater experience. Moreover, by making the learning process more efficient meta-learning targets the aforementioned issues of data and compute scarcity.

Meta-learning methods can be categorized into three broad classes. In \cite{grefenstette2019generalized}, authors introduce a unifying framework that encapsulates a wide class of existing approaches. 

Optimizer-focused methods aim to improve the underlying optimizer in the inner loop used to solve the tasks at hand by meta-learning optimizer initialization or hyperparameters. Within few-shot learning, Model Agnostic Meta Learning (MAML) and its variants use prior learning experience to meta-learn a model/policy initialization that requires just a few inner gradient steps to adapt to a new task \cite{maml, howtotrainmaml}. Other works aim to meta-learn optimizer hyperparameters. In \cite{lrmtl, improvemaml}, authors attempt to identify optimal learning rate scheduling strategies. Another strategy within this category is to directly learn a parameterization of the optimizer. Due to the sequential structure of inner loop parameter updates, recurrent architectures have been considered in this space \cite{Hochreiter2001LearningTL, learningtolearnbygd}. The inner loop optimization has also been viewed as a sequential decision-making problem and consequently optimizers have also been characterized as policies within an RL setting \cite{learningtooptimize}. 

Black-box methods represent the inner loop via a forward-pass of a single model. The learning process of the inner loop is captured by the activation layers of the underlying model. The inner loop learning can be instantiated as RNNs \cite{Ravi2016OptimizationAA, l2lbygdoriginal}, convolutional neural networks (CNN) \cite{Mishra2017ASN} or hyper-networks \cite{hypernetwork}. The meta-learning loop finds the hyperparameters of the inner loop network yielding good performance. 

In non-parametric methods, the inner loop aims to identify feature extractors that enable the matching of validation and training samples to yield an accurate prediction using the matched training label. The meta-loop aims to identify the class of feature extractors that transform the data samples into an appropriate space where matching is viable \cite{NIPS2016_90e13578, snell}. 

\subsubsection{Implicit Layers}\label{sssec:implicit}
Recent work has proposed a novel viewpoint wherein deep learning can be instantiated using implicit prediction rules rather than as conventional explicit feedforward architectures \cite{ghaoui2020implicit, chen2018neural, bai2019deep}. In \cite{ghaoui2020implicit} and \cite{bai2019deep} authors formalize how deep equilibrium models, characterized by nonlinear fixed point equations, represent weight-tied, infinite-depth networks. In this framework, \cite{ghaoui2020implicit} demonstrates how the aforementioned  models are able to generalize most of the popular deep learning architectures. In \cite{chen2018neural}, authors propose neural ordinary differential equations (ODE): an alternative instantiation of an implicit layer where the layer output is the solution to an ODE. This is shown to be an expressive model class yielding particularly impressive results when processing sequential data. Implicit layers have also been characterized as differentiable optimization layers. The work \cite{amos2017optnet} introduces differentiable quadratic optimization (QP) layers that can be incorporated within deep learning architectures. In \cite{cvxpylayers2019} authors develop software to differentiate through defined convex optimization problems. Some notable applications of differentiable optimization layers include parameterizing model predictive control policies \cite{NEURIPS2018_ba6d843e} and representing a maximum satisfiability (MAXSAT) solver \cite{Wang2019SATNetBD}. 
\subsection{Notation}\label{ssec:notation}
Throughout this work, we consider the problem of solving a task $\task[]$ which consists of an optimization problem and an evaluation step. The optimization problems are characterized with a loss function $\losst[](\theta)$ over decision variables $\theta$ belonging to some parameter space $\Theta\subseteq \R^p$. For evaluation, we denote a validation criterion $\lossv[]$ that assesses the optimizer $\theta^\star$ found in the associated problem. We refer to solving the optimization over $\theta$ as the \textit{inner loop} problem. At the meta-level we consider an algorithm $\Alg(\ \cdot\ ; \phi):\Theta\rightarrow \Theta$ with meta-parameters $\phi$ which generates a sequence of parameter updates using first-order information to solve the inner loop problem. We denote the horizon of the parameter update sequence by $T$. By meta-training, we refer to the optimization over the meta-parameters over a training set of $N$ tasks $\{\task\}_{i=1}^N$. For a vector-valued function $f(x): \R^d\rightarrow \R^p$ we let the operator $\nabla_x f(\cdot): \R^d\rightarrow \R^{d\times p}$ denote the gradient. If $f: \R^d \rightarrow \R$ is a scalar function, the Hessian of $f$ is denoted by $\nabla^2_x f(\cdot): \R^d \rightarrow \R^{d \times d}$. We denote that a square symmetric matrix $A$ is positive definite (all eigenvalues strictly positive) by $A \succ 0$. The vectorization of a matrix $A \in \R^{m \times n}$ is denoted by $\matvec(A) \in \R^{mn}$ and is constructed by stacking the columns of $A$. The Kronecker product of two matrices $A$ and $B$ is denoted $A \otimes B$. For a vector $x\in\R^n$ and $p\geq 1$, $\|x\|_p$ denotes the $\ell_p$-norm of $x$. For $m \in \mathbb{N}_+$, we define $[m]$ to be the set $\{a \in \mathbb{N}_+ \mid a \leq m\}$, where $\mathbb{N}_+$ is the set of positive integer numbers. We define the operator $\odot$ as an elementwise multiplication. $\mathcal{U}(a, b)$ denotes the uniform probability distribution with support $[a, b]$ and $\mathcal{N}(\mu, \sigma^2)$ represents a univariate normal distribution centered at $\mu$ with standard deviation $\sigma$. Finally, we define $E_{\mathcal{D}}[\cdot]$ as the expectation operator  over distribution $\mathcal{D}$.

\section{Background}\label{sec:background}
This section contextualizes our proposed framework. Section~\ref{ssec:firstordermethods} illustrates how conventional first-order update rules can be typically expressed as the solution to a convex optimization problem. Section~\ref{ssec:optlayers} then elaborates on the differentiable convex optimization methods that can be used to differentiate through the aforementioned inner loop gradient steps to update meta-parameters.

\subsection{First-order methods}\label{ssec:firstordermethods}
We consider a generic unconstrained optimization problem
\begin{align}\label{eq:optbasic_convex}
    \min_{x} f(x)
\end{align}
with differentiable objective $f$. First-order methods are a popular means to solve optimization problems of the form \eqref{eq:optbasic_convex}. The first-order property refers to the underlying methods' use of gradient information to generate a sequence of parameter iterates. Next we briefly survey a subset of important first-order methods that solve optimization problems of the form \eqref{eq:optbasic_convex}. We highlight how the update rules of these algorithms can be formulated as convex optimization problems themselves. This motivates the formulation of a generic parameterized convex optimizer to yield optimal parameter updates. 

\subsubsection{Gradient descent}\label{sssec:gd}
Gradient descent (GD) is a standard first-order method used to solve a variety of unconstrained optimization problems. For an unconstrained optimization problem, GD updates aim to reconcile the notion of minimizing a linear approximation of the objective while simultaneously maintaining proximity to the current parameter iterate. This can be cast as the convex optimization
\begin{align}\label{eq:gd_update}
    x^{(t+1)} = \argmin_x \{\nabla f(x^{(t)})^\top (x-x^{(t)})+\frac{\lambda}{2}||x-x^{(t)}||_2^2\}
\end{align}
where $\lambda>0$ is the step size. Solving \eqref{eq:gd_update} in closed-form yields the well-known GD update.

\subsubsection{Gradient descent with Momentum}\label{sssec:gdmomentum}
A popular practical variation of GD is to utilize the history first-order information within the parameter update rule. This is referred to as GD with momentum. The contribution of historic first-order information is captured by the notion of a state. More concretely, we define state update for $t>1$ as
\begin{align}
    \hist{}{t+1} = \beta \hist{}{t} + (1-\beta) \nabla f(x^{(t)}),    
\end{align}
where $\beta\in [0, 1]$ is an averaging parameter and we initialize $\hist{}{1}:=\nabla f(x^{(1)})$. The convex update rule in this method substitutes $\nabla f(x^{(t)})$ with $\hist{}{t+1}$:
\begin{align}\label{eq:gdmomentum_update}
    x^{(t+1)} = \argmin_x \{(\hist{}{t+1})^\top (x-x^{(t)})+\frac{\lambda}{2}||x-x^{(t)}||_2^2\}
\end{align}

Other notable first-order methods whose updates are defined via convex optimization problems are the proximal gradient (PG) \cite{boyd2004convex, Wright-Ma-2021} and mirror descent (MD) \cite{BECK2003167, doi:10.1137/1027074} methods. The former addresses unconstrained nondifferentiable problems whose objective is a composite function that can be decomposed into the sum of a differentiable and nondifferentiable part. The latter targets potentially constrained problems with updates that simultaneously minimize a linear approximation of the objective and a proximity term between parameter updates.

\subsection{Differentiable Optimization Layers}\label{ssec:optlayers}
We now present the formulation for a general DCO \cite{cvxpylayers2019}:
\begin{align}\label{eq:dco}
    D(x ; \phi):=\argmin_{y\in\R^n}\quad &f_0(x, y;\phi)\nonumber\\
    \textrm{s.t.}\quad &f_i(x,y;\phi)\leq 0 \quad \textrm{for } i\in[q],\nonumber\\
    &g_j(x,y;\phi) = 0 \quad \textrm{for } j\in[r],
\end{align}
where $x\in\R^d$ is the optimization input and $y\in\R^n$ is the solution. Here optimization parameters are defined by a vector $\phi$. The functions $f_i$ parameterize inequality constraint functions which are convex in $y$ and $g_j$ parameterize affine equality constraints. As with the constraint functions, the objective $f_0$ is convex in the optimization variable $y$.

Note that this formulation defines a general parameterized convex optimization problem in the output $y$. The solution to the optimization is a function of the input $x$.

When embedding DCO as a layer within the deep learning context, we require the ability to differentiate through $D$ with respect to $\phi$ when performing backpropagation. This is achieved via implicit differentiation through the Karush-Kuhn-Tucker (KKT) optimality conditions as proposed in \cite{amos2017optnet, amos2017icnn}. Particular instantiations of DCO, such as parameterized QPs, can enable more efficient backpropagation of gradients \cite{amos2017optnet}. 
\section{Meta-Optimization Framework}
Consider the setting where we have $N$ \textit{training} tasks $\task = (\losst, \lossv)$ for $i\in [N]$, where each task consists of a tuple containing a training loss function $\losst$ and an associated performance metric $\lossv$. Each of these tasks is sampled from an underlying distribution $\mathcal{D}$, i.e $\task \sim \mathcal{D}\ \forall i\in[N]$. For task $\task$, we consider the optimization 
\begin{align}\label{eq:optbasic}
    \min_{\theta_i \in\Theta} \losst(\theta_i)
\end{align}
where we aim to minimize loss $\losst$ over the decision variable $\theta$ constrained to the set $\Theta\subset \R^p$. We let $\phi$ denote the set of meta-parameters that configure the method used to solve optimization \eqref{eq:optbasic}, e.g. $\phi$ could include the learning rate in a gradient-based algorithm. The validation loss $\lossv$ is used to evaluate the final $\theta^\star_i$ recovered from solving \eqref{eq:optbasic}. As motivation for this setup, we consider the general training-validation procedure seen in ML. Here $\losst$ can be seen as the loss on training data with respect to model parameters $\theta$ and $\lossv$ denotes the loss on validation data. Note that for problems where the metric of interest is in fact the objective of \eqref{eq:optbasic}, we can trivially define $\lossv:=\losst$.
                                 
In this meta-learning framework, the goal is to perform well on a new task $\tasktarget = (\lossttarget, \lossvtarget) \sim \mathcal{D}$ using previous experience from tasks $\tasks$. Since $\tasktarget$ is sampled from the same distribution $\mathcal{D}$ as the training tasks, it has structural similarities that can be exploited by meta-learning.

\subsection{Inner Optimization Loop}\label{ssec:training_loop}
Depending on the structure of \eqref{eq:optbasic}, several iterative methods exist to solve the considered problem. The chosen algorithm has an update rule that yields a sequence of parameter updates $\{\thetait{i}{t}\}_{t=1}^{T}$ where $T$ is defines the total number of updates and $i$ indexes the associated task $\task$. Within the class of first-order methods, these update rules require computing or estimating (e.g. within reinforcement learning) the gradient $\grad{i}{t}:=\nabla_{\theta}l_i(\thetait{i}{t})$ to solve the inner optimization of task $\task$. The algorithm $\Alg$ applies the computed first-order and zeroth-order information at time step $t$ along with an abstraction of past information encapsulated by state $\hist{}{t}$ to yield both an updated parameter $\thetait{i}{t+1}$ and state $\hist{}{t+1}$: 


\begin{align}\label{eq:update_rule}
(\thetait{i}{t+1}, \hist{i}{t+1}) :=\Alg(\thetait{i}{t}, \hist{i}{t}, \grad{i}{t}; \phi).
\end{align}

Here we characterize the optimizer with meta-parameters $\phi$. Solving the inner loop problem to completion involves recursively applying \eqref{eq:update_rule} $T$ times from an initial condition $\thetait{i}{1}$ and history state $\hist{i}{1}$, which we denote by $\Alg_T(\thetait{i}{1}, \hist{i}{1}; \phi)$. Note that moving forward, unless made explicit, we suppress the return argument of the next state, i.e. we utilize the shorthand $\thetait{i}{t+1} :=\Alg(\thetait{i}{t}, \hist{i}{t}, \grad{i}{t}; \phi)$.

\subsection{Meta-Learning Loop}\label{ssec:metatraining_loop}
The meta-learning loop wraps around the inner loop. It aims to find optimal meta-parameters $\phi^\star$ that ensure that for each task $\mathcal{T}_i$ in distribution $\mathcal{D}$, the inner loop optimizer $\Alg$ produces $\theta_i^\star$ that performs well on metric $l_i^{\textrm{val}}(\theta_i^\star)$:

\begin{align}\label{eq:training}
    \min_{\phi} E_{\task\sim\mathcal{D}}[ \lossv(\theta_i^\star(\phi))]
\end{align}
where $E_{\task \sim\mathcal{D}}$ denotes the expectation over task distribution $\mathcal{D}$. An empirical version of this meta-learning process with training tasks $\tasks$ can be formulated as

\begin{align}\label{eq:meta_opt}
    \phi^\star &= \argmin_{\phi} \frac{1}{N}\sum_{i=1}^N \lossv(\theta_i^\star) \nonumber\\
    &= \argmin_{\phi} \frac{1}{N}\sum_{i=1}^N \lossv(\argmin_{\theta\in\Theta} \losst(\theta)) \nonumber\\
    &\approx \argmin_{\phi} \frac{1}{N}\sum_{i=1}^N \lossv(\Alg_T(\thetait{i}{1}, \hist{i}{1}; \phi)), \\
    &\defeq  \argmin_{\phi} \losstotal(\phi)
\end{align}
where the inner optimization is approximated by running algorithm $\Alg$ for $T$ time steps. Optimization \eqref{eq:meta_opt} can be approximated by another iterative gradient-based scheme that estimates $\nabla_{\phi} \lossv(\theta_i^\star)$. This requires differentiation through the inner loop update rule 
$\Alg$ with respect to meta-parameters $\phi$. More specifically, we require differentiation with respect to $\phi$ through a trajectory of parameter updates with horizon $T$. The meta-parameters will then be updated using a meta-optimizer of choice that uses first-order information on the meta-parameters:

\begin{align}\label{eq:meta_update_rule}
\phit{t+1} :=\MetaAlg\left(\phit{t}, \nabla_{\phi} \losstotal\left(\phit{t}\right)\right).
\end{align}

\begin{remark}\label{remark:practicalmetalearning}
Note that an approximated attempt at meta-learning is ubiquitous in practice. More specifically, the notion of hyperparameter selection (e.g. learning rate) for a first-order method is an instance of approximated meta-learning. In this context, we let hyperparameters be viewed as meta-parameters. Given a task, the goal in hyperparameter selection is to identify these such that the algorithm $\Alg$ generates $\theta_i^\star$ with low $\lossv(\theta_i^\star)$. In practice, the selection of hyperparameters (i.e. $\MetaAlg$) is restricted to crude rules of thumb or grid search guided by previous experience of similar problems. It is clear how such approximations can often fall short especially when considering high-dimensional or even continuous meta-parameter search spaces. Moreover, it does not accommodate parameterizing $\Alg$ to describe novel update rules. The meta-learning framework in \eqref{eq:meta_opt} generalizes the hyperparameter selection problem and makes it more rigorous.
\end{remark}

\subsection{Meta-Training}\label{ssec:meta-training}
The meta-training algorithm for an arbitrary optimizer $\Alg$ with meta-parameters $\phi$ is presented in Algorithm \ref{alg:method}. For each meta-parameter update, average validation losses across training tasks $\tasks$ are accumulated in $\losstotal$. For each task $\task$, these validation losses are measured after running the inner loop optimization using $\Alg(\cdot; \phi)$ for $T$ iterations. $\MetaAlg(\cdot, \cdot)$ then uses first-order information on $\losstotal$ with respect to $\phi$ to update the meta-parameters.
\begin{remark}\label{remark:horizon}
    For a specific task $\task$, the role of the meta-optimizer can be viewed as trying to learn the loss landscape of the inner problem locally around $\thetait{i}{t}$ for $t\in[T]$ and adapt the optimizer accordingly to encourage efficient descent. Thus, the updates within the inner loop help the meta-optimizer get a better gauge of the loss landscape. In turn, $T$ should be selected based on how complicated or spurious the inner problem's loss landscape is. For more complicated inner problems, more information (i.e. updates) are necessary to gauge the loss landscape. For simpler problems, a smaller horizon should suffice.
\end{remark}
\begin{remark}\label{remark:stochastic}
    Algorithm \ref{alg:method} allows for flexibility when choosing $\MetaAlg$. Stochastic first-order methods can be employed to solve the meta-training problem. That is, rather than using the entire batch of $N$ tasks $\tasks$, a random minibatch can be selected to perform a meta-parameter update. This strategy becomes particularly useful in settings where $N$ is prohibitively large. Furthermore, adding stochasticity in the $\MetaAlg$ procedure may reap some known benefits of SGD such as not succumbing to local minima. 
\end{remark}

\begin{algorithm}[tb]
   \caption{Meta-Training Framework}
   \label{alg:method}
\begin{algorithmic}
   \STATE {\bfseries Input:} Training set consisting of $N$ tasks $\tasks$
   \STATE {\bfseries Design choices:} Inner loop horizon $T$, meta-training epochs $M$, optimizer $\Alg(\cdot; \phi)$, meta-optimizer $\MetaAlg(\cdot, \cdot)$
   \STATE {\bfseries Return:} Meta-parameters $\phi$
   \vspace*{\baselineskip}
   \STATE {\bfseries begin training}
   \STATE 1. Initialize meta-parameters $\phi^{(1)}$
   \STATE 2. Initialize inner loop parameters and initial optimizer states $\{\thetait{i}{1}, \hist{i}{1}: i\in[N]\}$
   \FOR{$k\in[M]$}
   \STATE 3. Initialize $\losstotal \leftarrow0$
   \FOR{$i\in[N]$}
   \FOR{$t\in[T]$}
   \STATE 4. Compute inner loop gradient $\grad{i}{t}\leftarrow  \nabla_{\theta}l_i(\thetait{i}{t})$
   \STATE 5. $(\thetait{i}{t+1}, \hist{i}{t+1}) \leftarrow\Alg(\thetait{i}{t}, \hist{i}{t}, \grad{i}{t}; \phi)$
   \ENDFOR{}
   \STATE 6. $\losstotal \leftarrow \losstotal + \lossv(\thetait{i}{T+1})/N$
   \ENDFOR{}
   \STATE 7. $\phit{k+1} \leftarrow\MetaAlg(\phit{k}, \nabla_{\phi} \losstotal)$
   \ENDFOR{}
   \STATE {\bfseries end training} 
\end{algorithmic}
\end{algorithm}

\section{Differentiable Convex Optimizers}\label{sec:dco_meta}
We now propose various instantiations of the the inner loop optimization step \eqref{eq:update_rule} as differentiable convex optimizations. More generally, our proposed DCO meta-learning framework parameterizes optimizer $\Alg$ as a DCO introduced in \eqref{eq:dco}:
\begin{align}\label{eq:dco_meta}
    \Alg(\cdot; \phi):= D(\cdot; \phi).
\end{align}
As discussed in Section~\ref{ssec:firstordermethods}, this formulation contains a range of well-known first-order update rules as special cases.

To demonstrate the representational capacity of general DCOs as formulated in \eqref{eq:dco} within the meta-learning context, we focus on the subclass of unconstrained differentiable QPs. Note that this is a narrower subclass of DCO as we no longer have an arbitrary convex objective but rather a convex quadratic one. However, as we will demonstrate, this narrower formulation lends itself naturally to generalize the structure of update rules of existing gradient-based methods. While the formulations themselves admit closed-form solutions, we treat these as convex optimizations in our implementations to stay true to the DCO framework.  

\subsubsection{DCO Gradient}\label{sssec:DCOg}
We propose DCO Gradient based on the convex optimization \eqref{eq:gd_update} that encodes the vanilla GD update rule. The formulation discards the optimizer state $\hist{i}{t}$ and simply encodes the update rule:
\begin{align}\label{eq:DCO_gradient}
    \thetait{i}{t+1} \defeq \argmin_{\theta} \{\big(\grad{i}{t}\big)^\top \theta + \frac{1}{2}||\Lambda \odot (\theta - \theta_i^{(t)})||_2^2\},
\end{align}
where the parameterization is given by $\phi:=\Lambda\in\R^p$. In this formulation, learning the parameter $\Lambda$ can be viewed as optimizing the per-weight learning rate within vanilla GD.

\subsubsection{DCO General Descent (DCOGD)}\label{sssec:DCOgd}
We introduce a generalization of the previous approach that enables a general linear transformation of the update gradient:
\begin{align}\label{eq:DCO_general_descent}
    \thetait{i}{t+1} &\defeq \argmin_{\theta} \{\big( B \grad{i}{t} \big) ^\top \theta + \frac{1}{2} ||\theta - \theta_i^{(t)}||_2^2\},
\end{align}
where $\phi:=B\in\R^{p\times p}$.
\subsubsection{DCO Momentum (DCOM)}\label{sssec:DCOm}
Finally, we extend formulation \eqref{eq:DCO_gradient} to include momentum information:
\begin{align}\label{eq:DCO_momentum}
    \hist{i}{t+1} &\defeq M \odot \grad{i}{t} + (\ones - M) \odot \hist{i}{t}, \\
    \thetait{i}{t+1} &\defeq \argmin_{\theta} \{ \big(\hist{i}{t+1}\big) ^\top \theta + \frac{1}{2} ||\Lambda \odot (\theta - \theta_i^{(t)})||_2^2\},
\end{align}
where the parameterization is given by $\phi:=\{\Lambda, M\in\R^p\}$. Here, the DCO learns both the learning rate and the momentum averaging mechanism on a per-weight basis.

\section{Theory} \label{sec: theory}

We illustrate the potential of the DCO framework by analyzing the meta-learner process for a class of linear least-squares problems. Specifically, we let the tasks $\task$ be the least-squares problems
\begin{align} \label{eq:shiftls}
    \min_{\theta \in \R^p} \|X\theta - (y + X \Delta_i) \|_2^2,
\end{align}
where $X$ is a \emph{fixed} feature matrix with $X^{\top} X$ invertible and the regression targets vary using task-specific $\Delta_i$. We restrict our task-dependent regression target shifts to lie in the range space of $X$ for theoretical tractability and concreteness: note that each task assumes a shifted version of the same loss landscape, with the optimal weights also shifted by $\Delta_i$.

Namely, let $\theta' = (X^{\top} X)^{-1} X^{\top} y$ be the typical least-squares solution to \eqref{eq:shiftls} in the case where $\Delta_i=0$. It is then clear by inspection that $\theta^*_i = \theta' + \Delta_i$, and that the minimizing loss $\losst(\theta_i^*)$ is invariant to $i$; we thus denote the solution to \eqref{eq:shiftls} by $\lossopt$. Note that here we consider the case where training and validation data are identical for a particular task; i.e. $\losst = \lossv$. With some abuse of notation, our $k^{\textrm{th}}$ meta-optimization step target task loss
\begin{align} \label{eq:shiftlstarget}
   \lossttargetit \defeq \lossttarget(\Alg_1(\thetait{}{1}; \phit{k})) 
\end{align}
consists of an identically constructed task \eqref{eq:shiftls} with a distinct $\Delta$ and $\thetait{}{1}$. Concretely, we consider the performance on the target loss after one inner loop optimizer step using the meta-parameters $\phit{k}$ obtained from $k$ meta-optimization steps. Naturally, we expect that increasing both the number of meta-optimization steps $k$ and the number of training tasks $N$ should help reduce the target loss, ideally such that $\lossttargetit$ approaches the optimum $\lossopt$. This is formalized in Theorem~\ref{thm:main}.

\begin{theorem} \label{thm:main}
Consider executing Algorithm~\ref{alg:method} with the DCOGD optimizer \eqref{eq:DCO_general_descent} and $T=1$ on the set of shifted least-squares problems $\tasks$ introduced in \eqref{eq:shiftls}, each with an arbitrary but fixed initial parameter $\thetait{i}{1} \in \R^p$. Instantiate $\MetaAlg$ as standard GD with step size $\eta > 0$. Finally, define the set of vectors $\{Z_i\}_{i=1}^N$ by
\begin{align*}
    Z_i &\defeq \thetait{i}{1} - \Delta_i - \theta'.
\end{align*}
Then if $\{Z_i\}_{i=1}^N$ span $\R^p$, there exists a sufficiently small $\eta$ such that the one-step target task loss \eqref{eq:shiftlstarget}  approaches the optimum as the number of meta-steps $k \rightarrow \infty$; specifically,
\[
    \lossttargetit - \lossopt \leq O((1 - \epsilon)^k),
\]
for some $0 < \epsilon < 1$.
\end{theorem}
\begin{proof}
We can solve the gradient update from \eqref{eq:DCO_general_descent} in closed form. Doing this yields the following weight vector the $i$th task after one step on the inner problem:
\begin{align}
    \thetait{i}{2} \defeq \thetait{i}{1} - B \grad{i}{1}.
\end{align}

The total loss $\losstotal$ with $T=1$ can therefore be written as:
\begin{align}
    \losstotal &= \frac{1}{N} \sum_{i=1}^N \| X (\thetait{i}{1} - B \grad{i}{1}) - (y + X \Delta_i) \|_2^2 \nonumber \\
   &= \frac{1}{N} \sum_{i=1}^N \| X B \grad{i}{1} - X \thetait{i}{1} + y + X \Delta_i \|_2^2.  \label{eq:shiftlsloss}
\end{align}

The meta-learning problem aims to minimize this loss over the meta-parameter $\phi = B$. We will proceed with two steps: (1) show that there exists a meta-parameter $B^*$ for which $\losstotal$ equals the optimal minimizer $l^*$, and (2) show that $B^*$ is attained by our meta-learning procedure.

The existence of such a minimizing $B^*$ can be shown directly by letting $B^* = (1/2) (X^{\top} X)^{-1}$. Noting that
\[
    \grad{i}{1} = 2X^{\top} X (\thetait{i}{1} - \Delta_i) - 2X^{\top} y
\]
from differentiation of \eqref{eq:shiftls}, we can substitute $B^*$ and $\grad{i}{1}$ into \eqref{eq:shiftlsloss} to yield:
\begin{align}
    \losstotal = \frac{1}{N} \sum_{i=1}^N \| (X^{\top} X)^{-1} X^{\top} y - y \|_2^2 = \lossopt.
\end{align}

We now show that $B^*$ is attained by the meta-learning procedure. Namely, we consider our total loss for each outer meta-learning step in Algorithm~\ref{alg:method} to be a function $\losstotal(B)$ of our meta-parameter $\phi=B$. We let $\lossttarget(B)$ be defined similarly. It is easy to verify that the gradient $\nabla_{B} \losstotal$ is Lipschitz in $B$; therefore, we aim to show strong convexity of $\losstotal$ in $B$ to complete the proof using standard convex optimization results.

For convenience, define $y_i' \defeq -X \thetait{i}{1} + y + X \Delta_i$. Substituting into \eqref{eq:shiftlsloss}, we want to show that the following is strictly convex:
\[
    \losstotal(B) = \frac{1}{N} \sum_{i=1}^N \| X B \grad{i}{1} + y_i' \|_2^2.
\]

Expanding the square, scaling, and dropping terms which are linear in $B$ and thus do not affect convexity, we have that $\losstotal(B)$ is strictly convex iff $f(B)$ is strictly convex, where
\[
    f(B) = \sum_{i=1}^N (\grad{i}{1})^{\top} B^{\top} X^{\top} X B \grad{i}{1}.
\]

With some abuse of notation, we aim to compute the Hessian of $f(B^v)$ with respect to the vectorized $B^v = \matvec(B)$. Using standard matrix calculus identities \cite{Petersen2008} gives
\begin{align*}
    \frac{\partial f}{\partial B} &= 2 \sum_{i=1}^N X^{\top} X B \grad{i}{1} (\grad{i}{1})^{\top} \\
   &= 2 X^{\top} X B \sum_{i=1}^N \grad{i}{1} (\grad{i}{1})^{\top}.
\end{align*}

In order to compute the Hessian, we need to express $\matvec(\frac{\partial f}{\partial B}) = \frac{\partial f}{\partial B^v}$. Using the standard identity $\matvec(ABC) = (C^{\top} \otimes A) \ \matvec(B)$ yields
\begin{align*}
    \frac{\partial f}{\partial B^v} &= \left(\sum_{i=1}^N \grad{i}{1} (\grad{i}{1})^{\top} \otimes 2 X^{\top} X \right) B^v.
\end{align*}

Note that the gradient of $f$ with respect to $B^v$ is now linear in $B^v$. It is therefore immediate that that our desired Hessian is a constant matrix
\begin{align*}
    \nabla^2_{B^v} f &= \sum_{i=1}^N \grad{i}{1} (\grad{i}{1})^{\top} \otimes 2 X^{\top} X.
\end{align*}

Note that the Kronecker product of positive definite matrices is positive definite. Since $X^{\top} X \succ 0$ by assumption, $\nabla^2_{B^v} f \succ 0$ (and therefore $\losstotal(B)$ is strictly convex) if $\sum_{i=1}^N \grad{i}{1} (\grad{i}{1})^{\top} \succ 0$. This occurs iff the set of vectors $\{\grad{i}{1}\}_{i=1}^N$ spans $\R^p$. Invertibility of $X^{\top} X$ implies that this is equivalent to the collection of vectors $\{Z_i\}_{i=1}^N$ spanning $\R^p$, where $Z_i$ is as defined in the theorem statement.

Therefore $\nabla^2_{B^v} f \succ 0$, and $\losstotal(B)$ is strongly convex. Letting $B^{(k)}$ denote the values of the meta-parameters after $k$ iterations of GD, by standard convex optimization results \cite[Theorem~3.6]{garrigos2023handbook} we have that for a sufficiently small step size $\eta$,
\begin{align*}
    \| B^{(k)} - B^* \|_2^2 \leq O\left((1 - \epsilon)^k \right),
\end{align*}
where $0 < \epsilon < 1$. As $\lossttarget(B)$ is Lipschitz on any bounded set around $B^*$, the linear convergence in parameter space implies linear convergence in value, and we have shown the desired statement.
\end{proof}

Note that the condition that $\{Z_i\}_{i=1}^{\infty}$ span $\R^p$ is satisfied almost surely for typical random initializations of $\thetait{i}{1}$. Theorem~\ref{thm:main} can thus be interpreted as follows: provided at least $N=p$ sensibly initialized meta-training tasks, the meta learner will eventually learn to solve any target task to arbitrary precision with \emph{exactly one} inner-loop gradient descent step. This is an interesting formal guarantee that suggests the expressive power of our DCO meta-learning framework. While this section focuses on a particularly simple and tractable family of shifted least-squares problems as a proof-of-concept, we expect that the DCO meta-learning framework provides a tractable avenue for more sophisticated convergence results.
\section{Experiments}\label{sec:exp}
We verify the effectiveness of the proposed DCO meta-learning framework on some illustrative tasks. Specifically, we leverage the DCO optimizer instantiations introduced in Section \ref{sec:dco_meta} to solve linear least squares, system identification, and smooth function interpolation tasks.

\subsection{Meta-Training Setup}\label{ssec:meta_tr_instantiation}

Meta-parameters $\phi$ were initialized such that the DCO optimizers resemble existing first-order update rules. $\Lambda$ and $M$ were set to constant vectors in \eqref{eq:DCO_gradient} and \eqref{eq:DCO_momentum} to mimic the GD update rules introduced in Section \ref{ssec:firstordermethods}. Similarly, we initialized $B$ as the identity matrix in formulation \eqref{eq:DCO_general_descent}.
    
One potential challenge in training DCO optimizers is in ensuring that the proposed formulations remain well-posed for the entirety of the unrolling of the computational graph represented by Algorithm \ref{alg:method}. While the formulations as unconstrained QPs are by themselves well-posed, from a practical viewpoint potentially ill-posed inputs need to be handled. This is especially true for the initial meta-training epochs, where suboptimal meta-parameters may give rise undesirably small or large inner loop gradients. This was overcome by normalizing inner loop gradients before feeding them into the DCO optimizers.

In our experiments we set $T=1$ with $T$ as defined in Algorithm \ref{alg:method}. Restricting $T$ has the effect of explicitly training the DCO optimizer to perform an aggressive inner loop descent step. From a compututational standpoint, this restriction of $T$ allows to reallocate compute resources from solving several DCOs in the inner loop to performing more meta-parameter updates.

Note that Algorithm \ref{alg:method} allows for any first-order meta-optimizer to perform updates on $\phi$. For simplicity we restrict ourselves to using RMSProp with default hyperparameter settings as suggested in the \texttt{PyTorch} library. 

The DCO optimizers were implemented on a 2.2 GHz single-core CPU using the \texttt{CVXPYLayers} library \cite{cvxpylayers2019} and were solved using general-purpose interior-point solvers. While the implementation could be made more efficient, it suffices to outline the potential of the DCO meta-learning framework to outperform existing first-order baselines.

Throughout the experiments a comparison baseline of first-order methods Adam, SGD and RMSProp was considered due to their prevalence in solving unconstrained minimizations. For each baseline optimizer the learning rate was tuned and the best validation performance was reported.

\subsection{Least Squares Task}\label{ssec:ls}
 We first focus on solving least-squares (LS) problems
\begin{align}\label{eq:ls_problem}
    \min_{\theta\in \R^{100}} \|X\theta - y\|_2^2,
\end{align}
where $X\in\R^{100\times 100}$, $y\in\R^{100}$ with $X_{ij}, y_i\sim \mathcal{N}(0, 1)\ \forall i,j\in [100]$. The meta-training set was constructed by sampling 100 tasks according to \eqref{eq:ls_problem}. For each task, a LS objective was sampled which acts as both as $l_i$ and $\lossv$, i.e. $l_i=\lossv$ for $i\in[100]$. Meta-training was run for $M=20$ epochs. Then 100 new tasks were sampled and the evolution of the average loss across tasks over 30 training epochs was compared with existing first-order methods.  Figure \ref{fig:LS} shows the results. The DCO optimizers exhibit substantially faster convergence compared to classical baselines.

\begin{figure}
\centering
\includegraphics[width=0.4\textwidth]{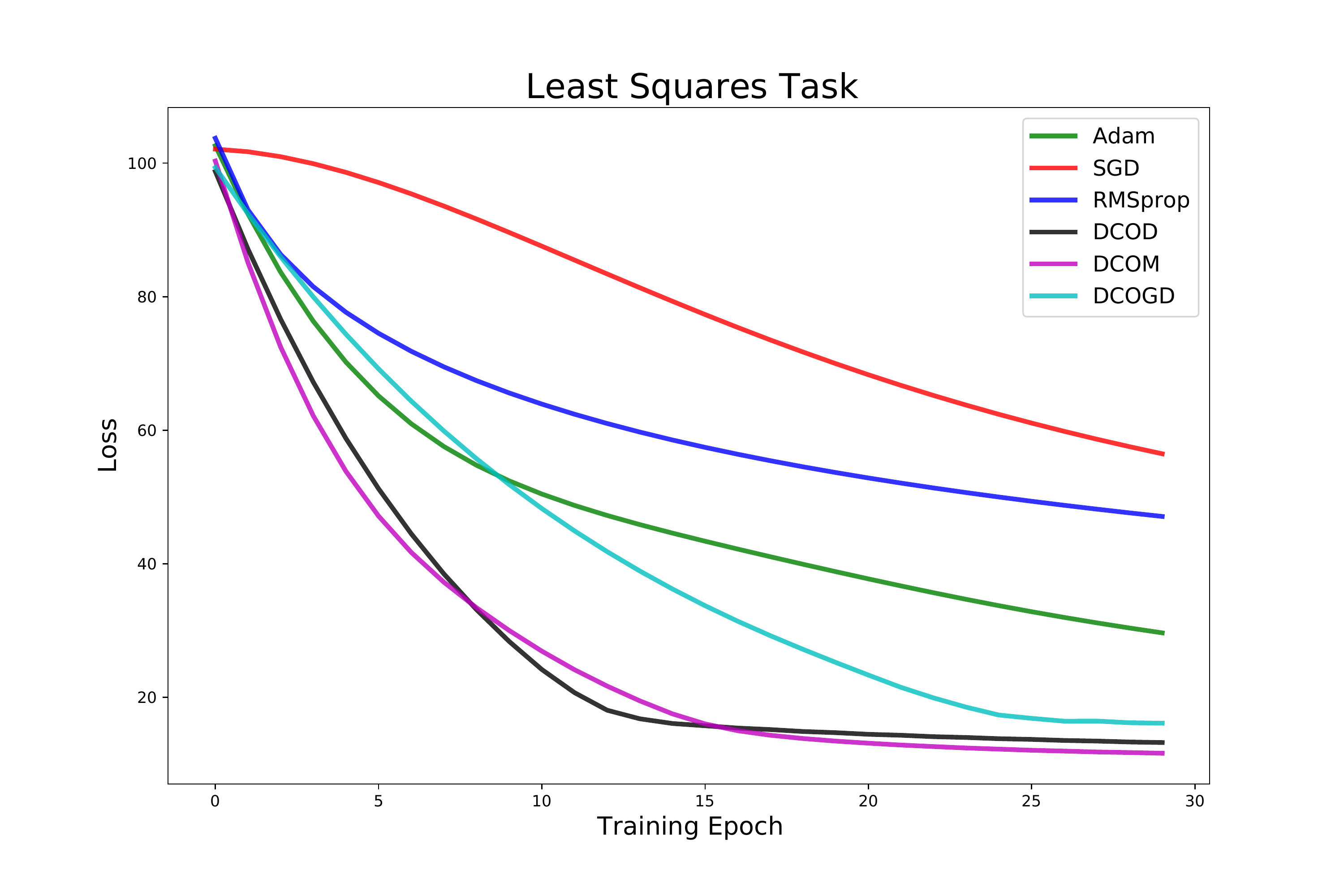}
\caption{Optimization performance on 100-dimensional LS tasks. Validation curves are averaged across 100 new tasks.}
\label{fig:LS}
\end{figure}

\subsection{System Identification Task}\label{ssec:system_id}
Next, we consider the task of identifying the underlying nonlinear discrete-time dynamics for population growth. We approximate the Beverton–Holt model given by
\begin{align}
    n_{t+1} = f(n_t):=\frac{R_0 n_t}{K+n_t},
\end{align}
where $n_t$ represents the population density in generation $t$, $R_0>0$ is the proliferation rate per generation, and $K>0$ is the carrying capacity of the environment. To introduce stochasticity into the model we include additive disturbance $d \sim \mathcal{N}(0, 0.1)$. In this context, we define a particular task by sampling a system with $R_0, K\sim \mathcal{U}(1, 2)$ and then generating training and validation samples $\{n, f(n)\}$ with $n\sim \mathcal{U}(0, 10)$. For each task we sample 500 training points and 100 validation points. The goal is to learn the underlying discrete nonlinear dynamics using a feedforward architecture with design ($1$-$5$-$5$-$1$), i.e. 2 hidden layers with 5 units each. The training of each network is carried out on the training set sampled for each task, and the final performance for that task is measured using the mean-square error (MSE) metric on the associated validation set. Meta-training was run for $M=20$ epochs. Figure \ref{fig:SI} presents the performance comparison between considered methods on 100 newly sampled tasks. The DCO optimizers continue to outperform baselines.

\begin{figure}
\centering
\includegraphics[width=0.4\textwidth]{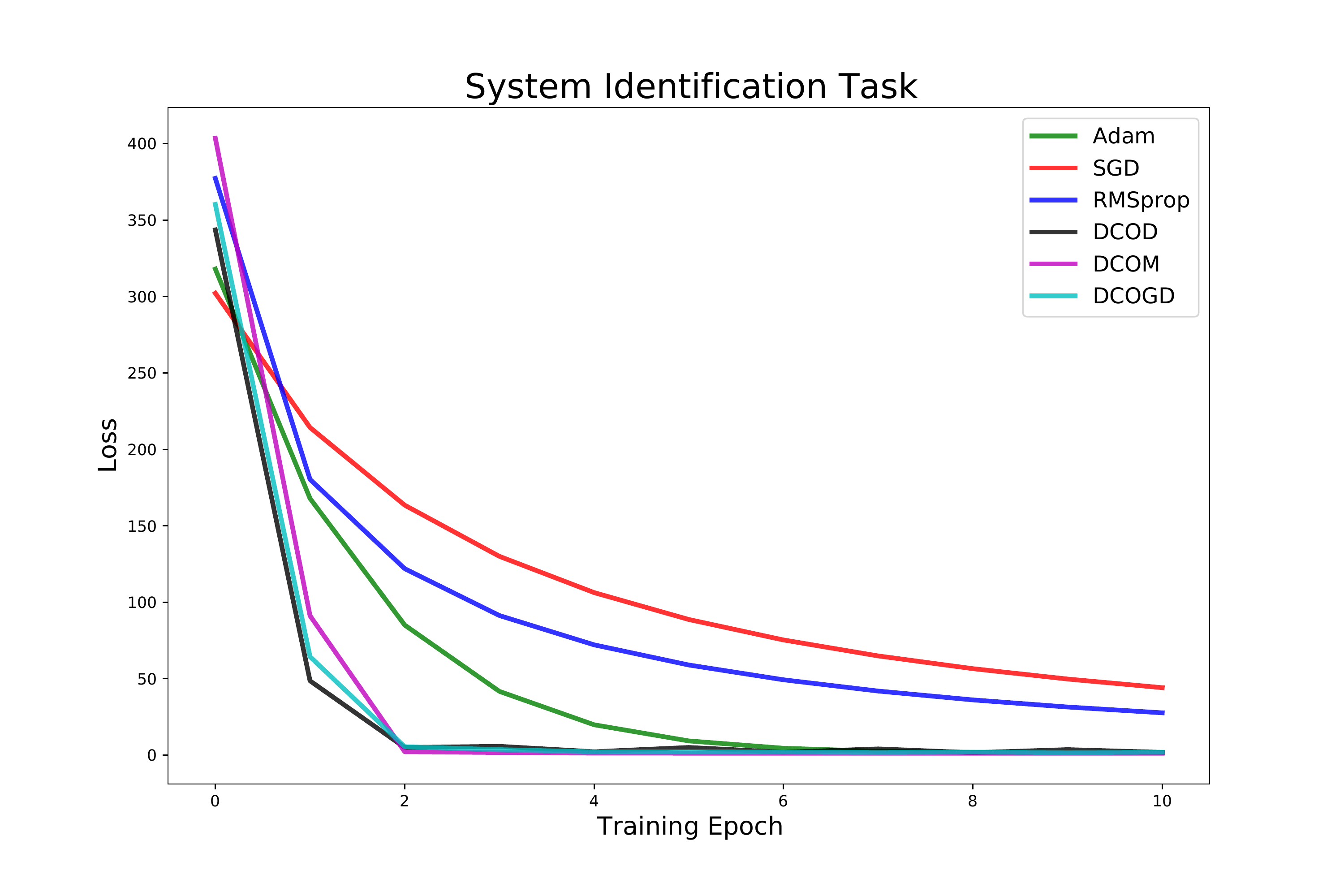}
\caption{Optimization performance on approximating the Beverton–Holt dynamics. Validation curves are averaged across 100 new tasks.}
\label{fig:SI}
\end{figure}

\subsection{Smooth Function Interpolation Task}\label{ssec:function_interpolation}
We finally consider the task of interpolating a real, nonlinear, smooth, univariate function via regression. As an illustrative example, we consider the smooth function
\begin{align}\label{eq:smoothfn}
    g(x) = a\cos(b x)\textrm{exp}(-c|x|)
\end{align}
where $a, b, c \sim \mathcal{U}(0, 1)$. A particular task is constructed by sampling 500 training points and 100 validation points from an instance of $g(x)$. The goal of the task was to learn a feed-forward network (FFN) with architecture ($1$-$10$-$10$-$1$) consisting of $2$ hidden layers with $10$ units each that yields low validation loss. As before, meta-training was run for $M=20$ epochs. The performance comparison with first-order methods on a new set of tasks is shown in Figure \ref{fig:FI}. The validation loss learning curves are averaged over 10 tasks. Similar to previous settings, we have obtained an improved convergence of DCO optimizers over baselines.
\begin{figure}
\centering
\includegraphics[width=0.4\textwidth]{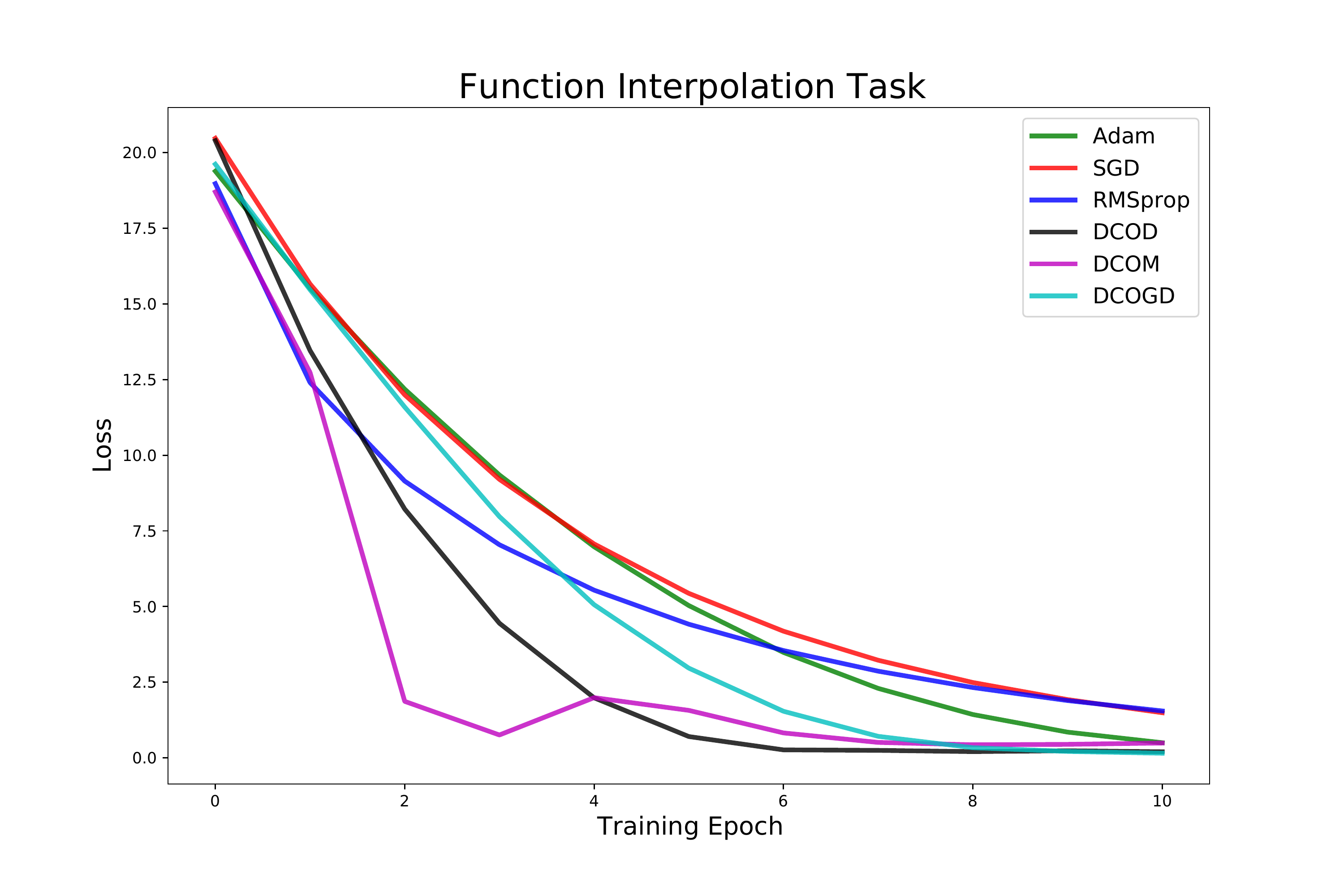}
\caption{Optimization performance on smooth interpolation tasks. Validation curves are averaged across 10 new tasks.}
\label{fig:FI}
\end{figure}

\section{Conclusion}\label{sec:conclusion}
This work introduces a novel DCO-based approach for optimizer design within the context of meta-learning. The DCO meta-learning framework remains loyal to the inherent convex nature of existing first-order update rules. We demonstrate that DCO-based optimizers not only generalize existing first-order methods but also have the potential of representing novel update rules. Theoretically, we show rapid convergence to the optimal update rule when meta-training DCOGD optimizers for a family of linear least-squares tasks. Experimentally, we demonstrate faster convergence of the DCO instantiations as compared to existing first-order methods on a range of illustrative tasks. Exciting future work involves finding a more general instantiation of DCO optimizers and scaling this approach to more complex networks.

\bibliographystyle{IEEEtran}
\bibliography{References}

\end{document}